\documentclass[orivec]{llncs}

\usepackage{graphicx}
\usepackage{amssymb}

\usepackage[ansinew]{inputenc}

\usepackage{dsfont}

\newtheorem{teo}{Theorem}
\newtheorem{coro}{Corollary}
\newtheorem{lema}{Lemma}

\newtheorem{ejem}{Example}
\newtheorem{defi}{Definition}

\newtheorem{obs}{Observation}

\newcommand{\nc}{\newcommand}

\newcommand\Omit[1]{}

\newcommand{\ramon}[1]{}

\newcommand{\triche}{\vspace{-1.7mm}}

\newcommand{\Set}[1]{\{ {#1} \}}

\nc{\vecleximi}{\succeq^{lex}_ {max}}

\renewenvironment{proof}{\noindent \bf Proof:
\setlength{\parskip}{0pt} \rm}{\ \null
\null \ \hfill \rule[-1mm]{1.2mm}{2mm}\rule[-2mm]{0mm}{4mm} \\[-.5mm]}

\nc{\Endproof}{\null \ \hfill \rule[-1mm]{1.2mm}{2mm}\rule[-2mm]{0mm}{4mm} \\[-.5mm]}

\nc{\OI}{{\bf OI}}
\nc{\Pv}{{\bf P5}}
\nc{\Pu}{{\bf P1}}
\nc{\Pii}{{\bf P2}}
\nc{\Piii}{{\bf P3}}
\nc{\Piv}{{\bf P4}}
\nc{\LM}{{\bf LM}}
\nc{\RM}{{\bf RM}}
\nc{\Di}{{\bf D1}}
\nc{\Dii}{{\bf D2}}
\nc{\Diii}{{\bf D3}}
\nc{\Li}{{\bf Lex1}}
\nc{\Lii}{{\bf Lex2}}
\nc{\Liii}{{\bf Lex3}}

\nc{\LIi}{{\bf LI1}}
\nc{\LIii}{{\bf LI2}}
\nc{\LIiii}{{\bf LI3}}
\nc{\Mon}{{\bf Mon}}
\nc{\E}{{\bf Ext}}

\newcommand{\lit}{{\it Lit}}
\newcommand{\fl}{\rightarrow}
\newcommand{\cns}{C_{fc}}
\newcommand{\vac}{\varnothing}
\newcommand{\dm}{\diamond}

\newcommand{\conmu}{{\scriptstyle \triangle}}
\newcommand{\nat}{\mathds{N}}
\newcommand{\SA}{\mbox{\sf SA}}
\newcommand{\myFP}{\mbox{\sf FP}}

\newcommand{\SAu}{\mbox{\sf SA1}}
\newcommand{\SAd}{\mbox{\sf SA2}}
\newcommand{\SAt}{\mbox{\sf SA3}}
\newcommand{\SAc}{\mbox{\sf SA4}}
\newcommand{\SAci}{\mbox{\sf SA5}}
\newcommand{\SAs}{\mbox{\sf SA6}}
\newcommand{\SAsie}{\mbox{\sf SA7}}
\newcommand{\SAo}{\mbox{\sf SA8}}
\begin{document}

\title{\bf Exploring the rationality of some syntactic merging operators\\ (extended version)}%
\author{ José Luis Chacón \and Ram\'{o}n {Pino~Pérez}}
\institute{ \em
Departamento de Matem\'{a}ticas\\
  Facultad de Ciencias\\
  Universidad de Los Andes\\
   M\'{e}rida, Venezuela\\
\email{\{jlchacon,pino\}@ula.ve}}
\maketitle

\begin{abstract}
Most  merging operators are defined by semantics methods which have very high computational complexity. In order to have operators with a lower computational complexity,
some merging operators defined in a syntactical way have be proposed. In this work we define some syntactical merging operators and exploring its rationality properties.
To do that we constrain the belief bases to be sets of formulas very close to logic programs and the underlying logic is defined through forward chaining rule (Modus Ponens).
We propose two types of operators: arbitration operators when the inputs are only two bases and fusion with integrity constraints operators.
We introduce a set of postulates inspired of postulates LS, proposed by Liberatore and Shaerf and then we analyzed the first class of operators through these postulates.
We also introduce a set of postulates inspired of postulates KP, proposed by Konieczny and Pino Pérez and then we analyzed the second class of operators through these postulates.
\end{abstract}

  \section{Introduction}
 Belief merging   \cite{Rev93,Rev97,BKM91,BKMS92,Lin94,Lin95,Lin96,LS98,KP02} aims at  combining several pieces of information when there is no strict precedence between them, unlike
 belief revision \cite{AGM,Gar,KM,Gar2} where one combines two pieces of information one of which has higher priority. The agent faces several conflicting pieces of information coming from several sources of equal reliability\footnote{More generally the sources can have different reliability, but we will focus on the case where all the sources have the same reliability. There is already a lot to say in this case.}, and he has to build a coherent description of the world from them. One important aspect of belief merging, differentiating
 this theory from belief revision -even from non-prioritized belief revision (see \cite{Han98b})- is the fact that $n$ sources of information (with $n\geq 2$) are considered.

 This work is about belief merging in the framework of logic-based representation of beliefs. In this framework beliefs are sets of propositional formulas. Many merging operators have
 been defined in that setting (a complete  survey of  logic-based merging is \cite{KP11}). Most  merging operators are based in semantical representations and the computational complexity of the entailment problem is at least in the second level of the polynomial hierarchy. Precisely, the problem of deciding if a formula is entailed by the revised base  is in the class $\Pi^P_2$ \cite{EG92,Neb91} and the fact that belief revision operators are a particular case of belief merging operators \cite{KP02} say us that the last operators are
 complex at least concerning the entailment problem.

 In recent years there has been a growing interest in studying belief revision and merging in some particular fragments of propositional logic. In particular Horn clauses \cite{Delgrande08,LangloisSST08,DelgrandeW10,DelgrandeP11}, logic programs \cite{HuePW09,DelgrandeSTW09} and other more general definable fragments \cite{CreignouPPW12}.

 Some works have been done to define change operators in a syntactic way \cite{HueWP08,BJKP98}. One interesting feature of  operators defined in \cite{BJKP98} is that the computational complexity is polynomial. In that work
some revision operators and update operators are syntactically defined with a restriction of the language and the logic. Therein the only inference rule is Modus Ponens
 and the formulas are very close to clauses in Logic Programming.  But the semantics is classical and very simple and natural. Our representation of beliefs will be close to Logic Programs but with the natural and classical semantics. In the present work we study some operators of merging for which the beliefs have this simple representation.

 The first idea we explore is the use of a revision operator $\ast$ to define binary merging operator $\conmu$. The key point here is trying to emulate the following equation\footnote{The symbols $\varphi$, $\psi$ and $\mu$ (with subscripts if necessary) denote propositional formulas which are usually used to represent beliefs.}
 \begin{eqnarray}\label{guide_rev_comu}
 \varphi\conmu \psi = (\varphi\ast\psi)\vee(\psi\ast\varphi)
 \end{eqnarray}

  Notice that this kind of operators have been called arbitration operators by Liberatore and Schaerf  \cite{LS95,LS98}.
 Actually they  characterized these operators by a set of postulates of rationality. Here we adapt the postulates to the simple logic we use, we define syntactic arbitration operators and we study them to the light of Liberatore and Shaerf postulates modified.

 We explore afterwards  the idea of defining merging operators with integrity constrains using the following equation as a guide:
 \begin{eqnarray}\label{guide_fusion}
 \Delta_\mu(\varphi_1,\dots,\varphi_n)=\left\{
                                    \begin{array}{ll}
                                      \varphi_1\wedge \cdots\wedge \varphi_n\wedge\mu, & \hbox{if consistent;} \\
                                      \bigvee(\varphi_i\ast \mu), & \hbox{otherwise.}
                                    \end{array}
                                  \right.
 \end{eqnarray}
 We shall define syntactic merging operators with integrity constraints adapting the previous equation to the logic of forward chaining we use;
 we shall define the postulates characterizing merging operators (the natural adaptation of KP postulates \cite{KP02})   in this simple setting and then we shall analyze the satisfaction of these
 postulates     by our syntactic integrity constraint merging operators.

  For the two kinds of operators we will see, on one hand, that the computational complexity is polynomial. This is a considerable gain with respect to the more classical merging operators; on the other hand, some postulates are not satisfied. That is the compromise in order to have tractable operators: you gain in computational complexity but you lose some
  properties.

  The rest of the paper is organized as follows:
  Section \ref{preli} contains the basic definitions and the syntactic revision operators (first defined in  \cite{BJKP98}) used later. Section \ref{arbi} is devoted to
  study a syntactic arbitrage operator following the lines of Equation \ref{guide_rev_comu}. Section \ref{fusion}  is devoted to
    analyze the syntactic merging operators defined following the lines of Equation \ref{guide_fusion}. We finish with a section containing some concluding remarks.

 \section{Preliminaries}\label{preli}
 Let $\cal{V}$ be a finite set of propositional variables. The elements of $\cal{V}$  are called atoms.
 A literal (or fact) is an atom or the negation of an atom. The set
 of literals will be denoted \lit. A {\it rule} is a formula of the shape
 $l_1\wedge l_2\cdots\wedge l_n\fl l_{n+1}$ where $l_i$ is a literal for
 $i=1,\dots,n+1$. Such a rule will be denoted  $l_1, l_2,\cdots,l_n\fl l_{n+1}$;
 the part $l_1, l_2,\cdots,l_n$ is called the body of the rule and $l_{n+1}$ is called the head of the rule.
 A fact $l$ can be seen as a rule $\fl l$ with empty body.

 Let $R$ and $L$  be a finite set of rules with non empty body and a finite set of facts respectively.
 A program $P$ is a set of the form $R\cup L$. In such a case we will say that the elements of $R$ are the rules of $P$ and the elements of $L$ are the facts of $P$. The set of programs will be denoted {\it Prog}. Let $P=R\cup L$  be a program. We define the of consequences by forward chaining of
 $P$, denoted $\cns(P)$, as the smallest set of literals (with respect to inclusion)
  $L'$ such that:
 \begin{itemize}
    \item  [(i)]$L\subseteq L'$
    \item [(ii)]If $l_1,l_2,\dots,l_n\fl l$ is in $R$ and $l_i\in L'$ for $i=1,\dots,n$
    then $l\in L'$.
    \item [(iii)]If $L'$ contains two opposed literals (an atom and its negation) then $L'=\lit$.
 \end{itemize}

 A program $P$ is consistent if $\cns(P)$ does not contain two opposed literals (or alternatively $\cns(P)\not=\lit$), otherwise we say that
 $P$ is inconsistent and it will be denoted $P\vdash \bot$.
  Let
  $R$, $L$ and $L'$  be a finite set of rules and two finite sets of facts respectively.  $L$ is said to be  $R$-consistent if $R\cup L$ is consistent.
 $L$ is said to be $R\cup L'$ consistent if $L\cup(R\,\cup L')$ is consistent. Thus,
  if $P$ is a program, $L$ is $P$-consistent if $L\cup P$ is consistent.
 Let $L$ and $P$ be a set of literals and a program respectively. $L$ is said to be  fc-consequence of $P$ if  $L\subseteq\cns(P)$.

 It is important to note that in this setting the problem of the consistency is polynomial and the problem of a literal entailment is also polynomial.

 We can define a very natural hierarchy over $\cns(P)$ for a program  $P$. Let $L$ be $\cns(P)$. We define a partition of $L$, $L=L_0\cup L_1\cdots\cup L_r$, inductively in the following way:  $L_0$ is the set of facts of program  $P$; $L_i$ are the literals  $l\in L\setminus\{L_0\cup L_1\cdots\cup L_{i-1}\}$ such that there is a rule $l_1,l_2,\dots,l_k\fl l\in P$ satisfying  $\{l_1,\dots,l_k\}\subset L_0\cup L_1\cdots\cup L_{i-1}$. Thus,
  in $L_1$ we find the literals  not in  $L_0$ and  obtained by the rules of $P$ using the literals of $L_0$. In $L_i$ we find the literals  not in  $L_j$ with $0\leq j<i$
  obtained by the rules of $P$ using the literals of $L_0\cup\cdots\cup L_{i-1}$. Since  $P=R\cup L_0$, where $R$ is a finite set of rules and $L_0$ is a finite set of facts, it follows that  $|L|\leq|L_0|+|R|$, so each   $L_i$ is a finite set.

 \begin{ejem}
 Consider $P=\{a\fl b\,;\,a\fl c\,;\, b\fl t\,;\,c\fl s\,;\,t\fl s\,;\,s\fl w\,;\,u\fl h\,;\,a\,;\,u\}$ then
  $L_0=\{a,\,u\}$, $L_1=\{b,\,c,\,h\}$,  $L_2=\{t,\,s\}$ and  $L_3=\{w\}$.
  \end{ejem}

     The following result whose proof is easy will be useful later:
     \begin{lema}\label{lem1}
     Let $Q,P$ be programs, then $\cns(Q\cup P)=\cns(\cns(Q\cup P)\cup P)$.
     \end{lema}

Now we shall recall  some definitions of operators in \cite{BJKP98}.

 \begin{defi}[Exceptional sets of literals and rules]
 Let $P$ be a  program. A set of literals $L$ is said to be exceptional for $P$ if
 $L$ is not $P$-consistent. A rule $L\fl l$ in $P$ is exceptional for $P$ if $L$ is exceptional for
 $P$.
 \end{defi}
 
 \begin{obs}
 The point that makes all computations simple is  the fact that, given the definition of entailment and consistency we have,
   to know if a set of literals is $P$-consistent  is polynomial.
 \end{obs}

 Notice that if $P$ is not consistent all its rules are exceptional. The following hierarchy  of a program appears in \cite{Pea90}. It has been very useful (see for instance \cite{LM02}).

 \begin{defi}[Base]
    Let $P$ be a program. We define $(P_i)_{i\in w}$ a decreasing sequence of programs in the following way:
    $P_0$ is $P$ and $P_{i+1}$ is the set of exceptional rules of $P_i$. Since
     $P$ is finite, there is a first integer  $n_0$ such that for all $m>n_0$,
     $P_m=P_{n_0}$. If $P_{n_0}\neq\vac$ we say that $\langle
    P_0,\dots,P_{n_0},\vac\rangle$ is the base of $P$. If $P_{n_0}=\vac$, then
    $\langle P_0,\dots,P_{n_0}\rangle$ will be the base of $P$.
 \end{defi}

 For instance, if we take
 $P=\{a\fl b\,;\,b\fl\neg c\,;\,\neg c\fl\neg a\,;\,\neg c\fl b\,;\,\neg a\fl\neg b\,;\,\neg a\fl \neg c\}$,
 it is easy to see that all the rules of $P$ are exceptional, thus the base of $P$ is $\langle P,\emptyset\rangle$. Now we give another
 (classic taxonomic) example:

 \begin{ejem} Take $P=\{m\fl s\,;\,c\fl m\,;\,c\fl\neg s\,;\,n\fl c\,;\,n\fl s\}$ where
  $m,s,c,n$ represent  mollusk, shell, cephalopod and nautilus respectively. The base is $\langle P_0,P_1,P_2,P_3\rangle$ where
 $ P_0=\{m\fl s\,;\,c\fl m\,;\,c\fl\neg s\,;\,n\fl c\,;\,n\fl s\}$,
 $P_1=\{c\fl m\,;\,c\fl\neg s\,;\,n\fl c\,;\,n\fl s\}$,
 $P_2=\{n\fl c\,;\,n\fl s\}$ and $ P_3=\vac$
  \end{ejem}

 \begin{defi}[Rank function]
 Let $P$ be a  program and let  $\langle P_0,\dots,P_{n}\rangle$ be its base. We define
 $\rho(P,\cdot):Prog\fl \nat$, the rank function, as follows: $\rho(P,Q)=\mbox{min}\{i\in
 \nat:Q\;\mbox{is}\;P_i-\mbox{consistent}\}$ if $P$ and $Q$ are consistent, otherwise
 $\rho(P,Q)=n$.
 \end{defi}

 Notice that if $Q_1\subseteq Q_2$ then $\rho(P,Q_1)\leq\rho(P,Q_2)$.

 \begin{defi}[Rank revision]\label{revgra}
 Let $P$ and $P'$ be two  programs. We define the rank revision operator of $P$ by $P'$, denoted
  $P\circ_{rk}P'$, as follows:
 \[P\circ_{rk}P'=P_{\rho(P,P')}\cup P'\]
 \end{defi}
 That is, we take the new piece of information $P'$ together with the first program in the base (the least exceptional)
 which is consistent with $P'$.

 In order to generalize  this operator we need the define the hull of a program $P$ with respect another program $P'$.
 Let $I_P(P')$ the set of subsets of $P$ which are  consistent with
 $P'$, contain $P_{\rho(P,P')}$ and maximal with these properties. We define $h_P:Prog\fl{\cal P}(P)$ by letting $h_P(P')=\bigcap
 I_P(P')$. The computation of $I_P(P')$ is exponential.

 \begin{defi}[Hull revision]
 Let $P$ and $P'$ be two  programs. We define the hull revision operator of $P$ by $P'$, denoted
  $P\circ_{h}P'$, as follows:
 \[P\circ_{h}P'=h_P(P')\cup P'\]
 \end{defi}

 \begin{obs}
 It is easy to see, by the definitions, that $\cns(P\circ_{rk}P')\subseteq\cns(P\circ_{h}P')$. In such a case we say that
 $\circ_{h}$ is a conservative extension of
 $\circ_{rk}$. Actually, there are examples in which the inclusion is strict.
 \end{obs}

 In order to define the extended hull revision we define first what is a flock of programs. This is simply
 a vector $\langle Q_0,\dots, Q_n\rangle$ where each $Q_i$ is a program. We use the letters ${\cal A}$, ${\cal F}$
 (with subscripts if necessary) to denote flocks. We define the concatenation of flocks in the natural way:
 $\langle P_1,\dots,P_n\rangle\cdot\langle Q_1,\dots,Q_m\rangle\stackrel{def}{=}
 \langle P_1,\dots,P_n,Q_1,\dots,Q_m\rangle$. Suppose that  $\mathcal{A}$ is a flock, say $\mathcal{A}=\langle
 Q_1,\dots,Q_n\rangle$,  we define $\cns(\mathcal{A})=\bigcap_{\,i=1}^n\cns(Q_i)$. We identify
 a  program $P$ with the flock $\langle P\rangle$. With this identification flocks are more general objects than programs.
 We are going to define revision operators of flocks by programs in which the output will be a flock.
 First we consider   $I_P(Q)$ as a flock and we give the following definition:

 \begin{defi}
  Let $P$, $P'$ be two programs. 
  Let $I_P(P')$ be as before. We put
 \[P\circ_{eh}P'=\left\{\begin{array}{lcl}
 \langle H_1\cup P',\dots,H_n\cup P'\rangle& &\mbox{if}\;I_P(P')=\{H_1,\dots,H_n\}\\
 P'& &\mbox{if}\;I_P(P')=\vac
 \end{array}\right.\]
 More generally, if $\mathcal{A}=\langle  Q_1,\dots,Q_n\rangle$,  we define
 \[\mathcal{A}\circ_{eh}P=(Q_1\circ_{eh}P)\cdot(Q_2\circ_{eh}P)\cdots(Q_n\circ_{eh}P)\]
 where the symbol $\;\cdot\;$ denotes the concatenation of flocks.
 \end{defi}

 \begin{obs}
 With the previous definition $\circ_{eh}$ is a conservative extension of $\circ_{h}$, that is,
  $\cns(P\circ_{h}P')\subseteq\cns(P\circ_{eh}P')$. To see that, it is enough to notice that $h_P(P')\subseteq
 H_i$ for all $H_i\in I_P(P')$. For this reason the operator $\circ_{eh}$ is called
 {\it extended hull revision.}
 \end{obs}

 For a study of AGM postulates satisfied by the previous operators we suggest the reader see the work in \cite{BJKP98}.

 \section{Syntactic arbitrage operators}\label{arbi}

 Liberatore and Schaerf \cite{LS95,LS98} define  a kind of merging operators called arbitrage operators (or commutative revision operators). They consider  beliefs that are represented by propositional formulas built in a finite propositional language.
  Their operators, denoted $\dm$, map two belief bases in a new belief base and they are characterized by a set of postulates of rationality.
  Actually, Liberatore and Schaerf \cite{LS98} prove that if  $\circ$ is a revision operator (see \cite{KM91}) then $\dm$ defined by letting
 $\varphi\dm\mu=(\varphi\circ\mu)\vee(\mu\circ\varphi)$ is an arbitrage operator.
 
 Inspired by the previous ideas and having syntactic revision operators we are going to define syntactic arbitrage operators. We shall redefine the postulates of arbitrage in the syntactical framework and we shall analyse the relationships between our operators and these new postulates.

In order to define the syntactic operators and to state the new postulates, we must provide operators simulating over programs the conjunction and the disjunction over formulas respectively. Thus, the operator
 $\odot:Prog^2\fl{\cal P}(Lit)$ (simulating the conjunction) is defined by:
 \[
 P_1\odot P_2=\cns(P_1\cup P_2)
 \]
 And the operator $\oplus:Prog^2\fl{\cal P}(Lit)$ (simulating the disjunction) is defined by:
 \[P_1\oplus P_2=\cns(P_1)\cap\cns(P_2)\]
The operator $\,\odot\,$ over programs corresponds to operator $\wedge$ over formulas and
 the operator $\,\oplus\,$ over programs corresponds to operator $\vee$.

 Now we define three syntactic arbitrage operators of the shape $\dm:Prog^2\fl{\cal P}(Lit)$
 in the following way:
 \begin{eqnarray}\label{eq-arbi}
 P_1\dm P_2=(P_1\star P_2)\oplus(P_2\star P_1)
 \end{eqnarray}
 where $\star\in\{\circ_{rk},\circ_{h},\circ_{eh}\}$ with
 $\circ_{rk},\circ_{h},\circ_{eh}$ the rank revision, the hull revision and the extended hull revision respectively.
  Thus we dispose of three syntactic operators: $\dm_{rk},\,\dm_{h},\,\dm_{eh}$ called rank arbitration operator,
  hull arbitration operator and extended hull arbitration operator respectively.

 The following postulates are the natural translation of the postulates of Liberatore and Shaerf \cite{LS98} for arbitration operators to our framework
 (SA states Syntactic Arbitration):
 \[
 \begin{array}{lcl}
 (\mbox{\sf SA}1)&\quad&P\dm Q=Q\dm P\\
 (\mbox{\sf SA}2)&\quad&P\dm Q\subseteq P\odot Q\\
 (\mbox{\sf SA}3)&\quad&\mbox{If}\;P\odot Q\not\vdash\bot\;\mbox{then}\;P\odot Q\subseteq P\dm Q\\
 (\mbox{\sf SA}4)&\quad&P\dm Q\vdash\bot\;\mbox{if and only if}\;P\vdash\bot\;\mbox{and}\;Q\vdash\bot\\
 (\mbox{\sf SA}5)&\quad&\mbox{If}\;\cns(P_1)=\cns(P_2)\;\mbox{and}\;\cns(Q_1)=\cns(Q_2)\;
 \mbox{then}\;P_1\dm Q_1=P_2\dm Q_2\\
 (\mbox{\sf SA}6)&\quad&P\dm(Q_1\oplus Q_2)=\left\{
 \begin{array}{ll}
    P\dm Q_1 & \hbox{or} \\
    P\dm Q_2 & \hbox{or} \\
    (P\dm Q_1)\oplus(P\dm Q_2) & \hbox{} \\
 \end{array}
 \right.\\
 (\mbox{\sf SA}7)&\quad&P\oplus Q\subseteq P\dm Q\\
 (\mbox{\sf SA}8)&\quad&\mbox{If}\;P\not\vdash\bot\;\mbox{then}\;P\odot(P\dm Q)\not\vdash\bot
 \end{array}
\]

Postulate \mbox{\sf SA}1 states that the two pieces of information have equal priority. Postulate \SA2 says that ``conjunction" is stronger than arbitration.
Postulate \SA3, together with \SA2, say that under the consistency of the ``conjunction" of programs, such a ``conjunction"  is the result of arbitration of programs. Postulate \SA4
says that the only possibility for the inconsistency of arbitration is the inconsistency of each input. Postulate \SA5 is the equivalence of the syntax in our context.
Postulate \SA6 is the trichotomy postulate in our context. Postulate \SA7 says that arbitration is stronger than ``disjunction". Postulate \SA8 says that the output of arbitration
will be consistent with any consistent input.

 The following result summarizes the behavior of our arbitration operators with respect to the postulates above defined.
\triche
\begin{teo}
The operators $\dm_{rk},\,\dm_{h},\,\dm_{eh}$ satisfy the postulates \SAu, \SAd, \SAt, \SAc, \SAsie\ and \SAo. They don't satisfy \SAci\ nor \SAs.
\end{teo}

\begin{proof}

 \noindent $(\mbox{\sf SA}1)\quad$ It is straightforward because of the definition of $\dm_{rk},\,\dm_{h},\,\dm_{eh}$, since $P\dm Q=\cns(P\star Q)\cap\cns(Q\star P)$ and the intersection is commutative.\\[5mm]
     $(\mbox{\sf SA}2)\quad$Since $\cns(P\circ_{rk}Q)\subseteq\cns(P\circ_{h}Q)\subseteq\cns(P\circ_{eh}Q)$,
    it is enough to prove that $\dm_{eh}$ satisfies $(\mbox{\sf SA}2)$ to establish that $\circ_{rk}$ and $\circ_{h}$ satisfy also \SA2.
    Let $I_P(Q)=\{H_1,H_2,\dots,H_n\}$ and $I_Q(P)=\{T_1,T_2,\dots,T_m\}$ where each $H_i$
    is a maximal subset of $P$ containing $P_{\rho(P,Q)}$ and
    $Q$-consistent, and each  $T_j$  is a maximal subset of $Q$ containing
    $Q_{\rho(Q,P)}$ and $P$-consistent. By definition,
    $P\dm_{eh} Q=\cns(P\circ_{eh}Q)\cap\cns(Q\circ_{eh}P)$ where
     $\cns(P\circ_{eh}Q)=\bigcap_{i=1}^n\cns(H_i\cup Q)$ and
    $\cns(Q\circ_{eh}P)=\bigcap_{i=1}^m\cns(T_i\cup P)$.
    Since $H_i\subseteq P$, we have $H_i\cup Q\subseteq P\cup Q$;
    analogously  $T_i\cup P\subseteq Q\cup P$. By the monotonicity of $\cns$ we have
        $P\dm_{eh} Q\subseteq \cns(P\cup Q)=P\odot Q$. Thus, \SA2 is satisfied. \\[3mm]
     $(\mbox{\sf SA}3)\quad$ Suppose $P\odot Q\not\vdash\bot$, that is  $\cns(P\cup Q)\not\vdash\bot$, so
     $P\cup Q$ is consistent. Then $\rho(P,Q)=0$, that is
     $P_{\rho(P,Q)}=P$ and therefore $P\circ_{rk}Q=P\cup Q$. Analogously,
      $Q\circ_{rk}P=Q\cup P$; therefore
     $P\dm_{rk}Q=\cns(P\cup Q)=P\odot Q$. Thus,
     $P\odot Q\subseteq P\dm_{rk}Q\subseteq P\dm_hQ\subseteq P\dm_{eh}Q$. That is,
     Postulate $(\mbox{\sf SA}3)$ is satisfied by our three operators.\\[3mm]
     $(\mbox{\sf SA}4)\quad$ Suppose $P\dm_{eh} Q\vdash\bot$. We want to see that
     $\cns(P\circ_{eh}Q)\vdash\bot$ and $\cns(Q\circ_{eh}P)\vdash\bot$.
     Consider $I_P(Q)=\{H_1,H_2,\dots,H_n\}$ and $I_Q(P)=\{T_1,T_2,\dots,T_M\}$ where each $H_i$
    is a maximal subset of $P$ containing $P_{\rho(P,Q)}$ and
    $Q$-consistent, and each  $T_j$  is a maximal subset of $Q$ containing
    $Q_{\rho(Q,P)}$ and $P$-consistent. We have $\bigcap_{\,i=i}^{\,n}\cns(H_i\cup
    Q)$ is not consistent, therefore $\cns(H_i\cup Q)$ is not consistent; in particular
    $\cns(P_{\rho(P,Q)}\cup Q)$ is not consistent. By definition of $\circ_{rk}$, it follows that  $P_{\rho(P,Q)}=\vac$,
    Thus $\cns(Q)\vdash\bot$. With a similar argument, starting from  $\cns(Q\circ_{eh}P)\vdash\bot$ we obtain $\cns(P)\vdash\bot$.
    This proves one direction of Postulate \SA4, because if $P\dm_{rm}Q\vdash\bot$
    or $P\dm_hQ\vdash\bot$ necessarily $P\dm_{eh} Q\vdash\bot$.

    Conversely, suppose that $P$ and $Q$ are inconsistent. We want to see that
    $P\dm_{rm}Q\vdash\bot$,  $P\dm_hQ\vdash\bot$ and $P\dm_{eh} Q\vdash\bot$. As before
    it is enough to see that $P\dm_{rm}Q\vdash\bot$. By hypothesis, $\cns(P)=Lit$ and $\cns(Q)=Lit$, but
    $P\dm_gQ=\cns(P_{\rho(P,Q)}\cup
    Q)\cap\cns(Q_{\rho(Q,P)}\cup P)$.
     Moreover, $\cns(P_{\rho(P,Q)}\cup  Q)\supseteq\cns(Q)=Lit$ and $\cns(Q_{\rho(Q,P)}\cup P)\supseteq\cns(P)=Lit$.
     Therefore
    $P\dm_{rk}Q=Lit$.\\[3mm]
     $(\mbox{\sf SA}5)\quad$ Our syntactic arbitration operators don't satisfy this postulate. We build a counterexample.
     Define $P_1=\{a\fl
     c\,;\,b\},\;P_2=\{b\},\;Q_1=\{b\fl c\,;\,a),\;Q_2=\{a\}$. It is clear that $\cns(P_1)=\cns(P_2)=\{b\}$ and
      $\cns(Q_1)=\cns(Q_2)=\{a\}$; since $P_i\cup Q_i$  is consistent for $i=1,2$,
     we have, by $(\mbox{\sf SA}2)$ and $(\mbox{\sf SA}3)$,  $P_i\dm Q_i=\cns(P_i\cup Q_i)$ for
      $\dm\in\Set{\dm_{rk}, \dm_{h},\dm_{eh}}$. Notice that $P_1\cup Q_1=\{a\fl
     c;\,b\fl c;\,a;\,b\}$, thus $\cns(P_1\cup Q_1)=\{a,b,c\}$; also notice that $\cns(P_2\cup
     Q_2)=\cns(\{b,a\})=\{a,b\}$. Therefore our operators are syntax dependent.\\[3mm]
     $(\mbox{\sf SA}6)\quad$Our syntactic arbitration operators don't satisfy this postulate. We build a counterexample.
     Let's define $P=\{a\fl b\,;\,a\fl c\,;\,e\},\;Q_1=\{a\},\;Q_2=\{b\}$. Since $P\cup (Q_1\oplus Q_2)=P$,
     $P\cup Q_1=\{a\fl b\,;\,a\fl c\,;\,e\,;\,a\}$ and $P\cup Q_1=\{a\fl b\,;\,a\fl c\,;\,e\,;\,b\}$
     are  consistent, we have for $\dm\in\{\dm_{rk},\,\dm_h,\,\dm_{eh}\}$ the following equalities:
     $P\dm(Q_1\oplus Q_2)=\cns(P)=\{e\}$, $P\dm Q_1=P\odot Q_1=\{a,b,c,e\}$, $P\dm Q_2=P\odot Q_2=\{b,e\}$ and $(P\dm Q_1)\oplus(P\dm Q_2)=\{b,e\}$.
     Thus it is clear that
     $P\dm(Q_1\oplus Q_2)$ is different from the options in the postulate.\\[3mm]
     $(\mbox{\sf SA}7)\quad$ This postulate is verified by our three operators. Actually,
     $P\circ_{rk}Q=P_{\rho(P,Q)}\cup Q\supseteq Q$ and also $Q\circ_{rk}P\supseteq
     P$; from this it follows $\cns(P)\subseteq\cns(Q\circ_{rk}P)$ and
     $\cns(Q)\subseteq\cns(P\circ_{rk}Q)$ and by definition
     $P\oplus Q\subseteq P\dm_{rk}Q\subseteq P\dm_cQ\subseteq P\dm_{eh} Q$.
     \\[3mm]
     $(\mbox{\sf SA}8)\quad$ Assume that $\cns(P)\neq Lit$. Consider $I_Q(P)=\{\{T_1,T_2,\dots,T_m\}$
     where $T_i$ is a maximal subset of $Q$ containing
    $Q_{\rho(Q,P)}$ and $P$-consistent. Since $P$ is consistent,
      $\cns(T_i\cup P)\neq Lit$ for $i=1,\dots,m$.
    Notice that $P\dm_{eh}Q=(P\circ_{eh}Q)\oplus(Q\circ_{eh}P)\subseteq \bigcap_{i=1}^m
     \cns(T_i\cup P)$. Then, by  Lemma \ref{lem1},
     $P\odot (P\dm_{eh}Q)\subseteq \cns(P\cup \cns(T_i\cup P))=\cns(T_i\cup P)$. Therefore,
      $P\odot (P\dm_{eh}Q)\neq Lit$ and necessarily
      $P\odot (P\dm_{rk}Q)\neq Lit$ and $P\odot (P\dm_cQ)\neq Lit$.
     
     \end{proof}

     Now we illustrate the behavior of our three syntactic operators. In particular we shall see that they have behaviors well differentiated.

     \begin{ejem}
     Let $P$ and  $Q$ be two programs defined as follows:
     $
     P=\{a,\,b\fl\neg c\,;\,b\fl d\,;\,b\fl \neg c\,;\,\neg c\fl e\,;\,a,\neg c\fl f\,;\,a\}$ and
     $Q=\{a,\,b\fl c\,;\,a\fl e\,;\,a,\,e\fl c\,;\,a,\,e\fl d\,;\,c\fl d\,;\,c\fl f\,;\,b\}$. Then
     $\cns(P)=\{a\}$ and  $\cns(Q)=\{b\}$.
      It is easy to verify that  the bases of  $P$ and $Q$ have two levels.
     No rule is exceptional, thus $P_1=Q_1=\vac$.
      Since $P\cup Q$ is not consistent, we have
     $P_{\rho(P,Q)}=Q_{\rho(Q,P)}=\vac$. Therefore $I_Q(P)$ is the set of maximal subsets of
      $Q$ which are  $P$-consistent. With a little computation, we can verify that
     $I_Q(P)=\{T_1,\,T_2\}$ where $T_1=\{a\fl e\,;\,a,\,e\fl d\,;\,c\fl d\,;\,c\fl f\,;\,b\}$ and
       $ T_2=\{a,\,b\fl c\,;\,a\fl e\,;\,a,\,e\fl c\,;\,a,\,e\fl d\,;\,c\fl d\,;\,c\fl f\}$. Moreover
        $\cns(T_1\cup P)=\{a,b,\neg c,d,e,f\}$ and $\cns(T_2\cup P)=\{a,d,e\}$.
        For the same reasons as before,  $I_P(Q)$ is the set of maximal subsets of
        $P$ which are $Q$-consistent. This can be easily computed:
        $I_P(Q)=\{H_1,\,H_2\}$ where $H_1=\{b\fl d\,;\,\neg c\fl e\,;\,a,\neg c\fl f\,;\,a\}$ and
       $H_2=\{a,\,b\fl\neg c\,;\,b\fl d\,;\,b\fl \neg c\,;\,\neg c\fl e\,;\,a,\neg c\fl f\}$. Then,
     $\cns(H_1\cup Q)=\{a,b, c,d,e,f\}$ and $\cns\underline{}(H_2\cup Q)=\{b,\neg c,d,e\}$.
     Thus, $\cap I_Q(P)=\{a\fl e\,;\,c\fl d\,;\,c\fl f\}$ and
 $\cap I_P(Q)=\{b\fl d\,;\,\neg c\fl e\,;\,a,\neg c\fl f\}$. Therefore,
 $\cns(\cap I_Q(P)\cup P)=\{a,e,d\}$ and $\cns(\cap I_P(Q)\cup Q)=\{b,d\}\}$. Finally we can compute the outputs for the three operators.
   $P\dm_hQ=(P\circ_{h}Q)\oplus (Q\circ_{h}P)=
 \cns(P\circ_{h}Q)\cap\cns(Q\circ_{h}P)=\{d\}$.
  Since
  $P_{\rho(P,Q)}\cup Q=Q$ and $Q_{\rho(Q,P)}\cup P=P$, we have
 $P\dm_{rk}Q=\cns(P)\cap\cns(Q)=\vac$. For the last operator we have
  $P\dm_{eh}Q=(P\circ_{eh}Q)\oplus (Q\circ_{eh}P)=
 \cns(P\circ_{eh}Q)\cap\cns(Q\circ_{eh}P)$ where
 $\cns(P\circ_{eh}Q)=\cns(H_1\cup Q)\cap\cns(H_2\cup Q)=\{b,d,e\}$ and
  $\cns(Q\circ_{eh}P)=\cns(T_1\cup P)\cap\cns(T_2\cup P)=\{a,d,e\}$. Therefore
    $P\dm_{eh}Q=\{d,e\}$. Summarizing, we have:
    \[
 P\dm_{rk}Q=\vac \subset \Set d = P\dm_hQ\subset \Set{d,e}= P\dm_{eh}Q
    \]
     \end{ejem}

     \section{Merging programs} \label{fusion}

     Merging operators with integrity constraints were defined in  \cite{KP02}. Therein we can find a characterization in terms of postulates.
     The aim of this section will be to define merging operators for the representation of beliefs as programs as those presented in Section \ref{preli} with the logic of forward chaining
     defined therein. We shall also define the merging postulates in this syntactical framework and we shall study the relationships between our new operators and the postulates.

      We denote by $\it{Prog}$ the set of all the programs;
      $\cal{M}(\it{Prog})$ will denote the set of finite and nonempty  multisets of nonempty programs. The  merging operators  $\Delta$ we are interested in, are operators from
      $\cal{M}(\it{Prog})\times{\it Prog}$ into subsets of $\lit$. The multisets of programs are called profiles and we shall use the letters $\Phi$ and $\Psi$ (with subscripts if necessary) to denote them. If $\Phi=\Set{P_1,\dots,P_n}$ is a profile and $P$ is a program, $\Delta(\Phi,P)$ must be understood as the result of merging the programs in $\Phi$
      under the constraint $P$. We shall write $\Delta_P(\Phi)$ or $\Delta_P(P_1,\dots,P_n)$ instead of $\Delta(\Phi,P)$. For a profile $\Phi$ we denote $\cup \Phi$ the union of all programs in $\Phi$. If $\Phi_1$ and $\Phi_2$ are profiles we denote by $\Phi_1\sqcup\Phi_2$ the new profile resulting of the union of multisets (e.g. $\Set P\sqcup\Set P=\Set{P,P}$).

        Guided by Equation \ref{guide_fusion}, and the interpretation of ``disjunction" already defined, we set the following definition:
        \begin{eqnarray}\label{def-merging}
        \Delta_P(P_1,\dots,P_n)=\left\{
                                    \begin{array}{ll}
                                      \cns(P\cup P_1\dots\cup P_n), & \hbox{if this is consistent;} \\
                                      \bigcap\cns(P_i\circ_{rk} P), & \hbox{otherwise.}
                                    \end{array}
                                  \right.
        \end{eqnarray}
    where the profile is $\Set{P_1,\dots,P_n}$ (the programs to merge), under the integrity constraint $P$ and $\circ_{rk}$ is the rank revision. We adopt the following definition of entailment for programs $P$ and $Q$: we put $P\vdash Q$ whenever $\cns(Q)\subset\cns(P)$.

    The following  postulates are the adaptation of postulates characterizing merging operators with integrity constraints (see \cite{KP02}):

    \noindent Let $\Phi=\{P_1,\dots,P_n\}$.
     \begin{description}
     \item [ {(\myFP0)}]$\Delta_{P}(\Phi)\vdash P$
     \item [ {(\myFP1)}]If $P$ is consistent, then $\Delta_{P}(\Phi)$ is consistent.
     \item [ {(\myFP2)}]If $\cns(\cup \Phi\cup P)$ is consistent, then
     $\Delta_{P}(\Phi)=\cns(\cup \Phi\cup P)$.
     \item [ {(\myFP3)}]If $\Phi_1=\{P_1,\dots P_n\},\;\Phi_2=\{Q_1,\dots,Q_n\}$, $\cns(P_i)=\cns(Q_i)$, $\cns(P)=\cns(Q)$, then
     $\Delta_{P}(\Phi_1)=\Delta_{Q}(\Phi_2)$
     \item [ {(\myFP4)}] If $P_1\vdash P$, $P_2\vdash P$ and  $P_1$ and $P_2$ are consistent, then
     \[\cns(\Delta_{P}(P_1,P_2)\cup P_1)\neq\lit\Rightarrow\cns(\Delta_{P}(P_1,P_2)\cup P_2)\neq\lit\]
     \item [ {(\myFP5)}]
     $\Delta_{P}(\Phi_1)\cup\Delta_{P}(\Phi_2)\vdash\Delta_{P}(\Phi_1\sqcup\Phi_2)$
     \item [ {(\myFP6)}] If $\Delta_{P}(\Phi_1)\cup\Delta_{P}(\Phi_2)$ is consistent,
     then
     $\Delta_{P}(\Phi_1\sqcup\Phi_2)\vdash\Delta_{P}(\Phi_1)\cup\Delta_{P}(\Phi_2)$
     \item [ {(\myFP7)}] $\Delta_{P}(\Phi)\cup Q\vdash\Delta_{P\cup Q}(\Phi)$
     \item [ {(\myFP8)}] If $\Delta_{P}(\Phi)\cup Q$ is consistent, then
     $\Delta_{P\cup Q}(\Phi)\vdash\Delta_{P}(\Phi)\cup Q$
     \end{description}

     Postulate \myFP0 means that the integrity constraint is respected. Postulate \myFP1 establishes the consistency of the merging whenever the integrity constraint are consistent.
     Postulate \myFP2 establishes that in the case where there is no conflict between the pieces of information, the output is the consequence of putting all pieces together.
     Postulate \myFP3 is the independence of the syntax in our framework. Postulate \myFP4 is the postulate of fairness. Postulates \myFP5 and \myFP6 refer to a good behavior
     of the merging of subgroups: if the merging of subgroups agree on some facts and it is consistent, the result of the merging on the whole group is the set of facts on which the subgroups agree.  Postulates \myFP7 and \myFP8 concerning the iteration of the process (see \cite{KP02} for more detailed explanations about the postulates).

     The following theorem tells us how is the behavior of the operator defined by Equation~(\ref{def-merging}) with respect the postulates in the case of ranked revision.

\begin{teo}
The operator $\Delta^{rk}$ satisfies the postulates \myFP$0$, \myFP$1$, \myFP$2$\ and \myFP$4$. It doesn't satisfy \myFP$3$ nor \myFP$5$-\myFP$8$
\end{teo}

   \begin{proof}   
       (\myFP0) This postulate follows straightforward from the definition:  if $\cns(P\cup P_1\dots\cup P_n)$ is consistent, we have
     $\Delta_P(P_1,\dots,P_n)=\cns(P\cup P_1\dots\cup P_n)\supset\cns(P)$; therefore $\Delta_P(P_1,\dots,P_n)\vdash P$. If $\Delta_P(P_1,\dots,P_n)=\bigcap\cns(P_i\circ_{rk} P)$, by definition of rank revision  (see Definition \ref{revgra}), $P_i\circ_{rk} P=P_i^{\rho(P_i,P)}\cup P\supset P$. Thus, $\Delta_P(P_1,\dots,P_n)=\bigcap\cns(P_i\circ_{rk} P)\vdash P$.\\[3mm]
     (\myFP1) This postulate is straightforwardly verified, because  $\cns(P\cup P_1\dots\cup P_n)$ is consistent or by the definition of rank revision  $P_i\circ_{rk}P$ is consistent, so $\cap\cns(P_i\circ_{rk} P)$ is consistent.\\[3mm]
     (\myFP2) This postulate is straightforward by the definition.\\[3mm]
     (\myFP3) This postulate is not verified. We show a counterexample.
       Consider $\Phi_1=\{P_1\}$ and $\Phi_2=\{Q_1\}$ where $P_1=\{a\fl b\}$ and $Q_1=\{a\fl c\}$. It is clear that $\cns(P_1)=\vac=\cns(Q_1)$. Define $P=Q=\{a\}$.  Then $\Delta_P(P_1)=\cns(P\cup P_1)=\{a,\,b\}$ and $\Delta_Q(Q_1)=\{a,\,c\}$. Thus,
     $\Delta_P(P_1)\neq\Delta_Q(Q_1)$.\\[3mm]
     (\myFP4) We need some technical results in order to prove that this postulate is verified.

     \begin{lema}\label{lem2}
     Let $Q,P$ be programs such that  $Q$ is consistent, $Q\cup P$ is inconsistent and $Q\vdash P$, then $P\vdash Q\circ_{rk}P$.
     \end{lema}
     \begin{proof}
           Suppose there is $l\in\cns(Q\circ_{rk}P)$ and $l\notin\cns(P)$. Since $P\cup Q$ is inconsistent, we have $Q\circ_{rk}P=Q_i\cup P$ for $i\geq 1$.
      From the consistency of $Q$ follows that $Q_i$ has no facts. Let $L_0,\dots,L_r$ be the hierarchy of   $\cns(Q_i\cup P)$ (see discussion before Example 1). In $L_0$  there are only literals of $P$. Let $l$ be  minimal in the hierarchy of $\cns(Q_i\cup P)$ such that $l\notin\cns(P)$. We claim that
      there exists a rule $l_1,\dots,l_n\to l$ in $Q_i$ such that $\{l_1,\dots,l_n\}\subset \cns (P)$. To establish the claim we proceed as follows:
      since  $l$ is minimal such that $l\notin\cns(P)$, necessarily  $l\notin L_0$, because $L_0\subset\cns(P)$. Let $L_i$ be  such that $l\in L_i$ and $l\notin L_k$ for $k<i$.
      By the minimality of $l$, it is clear  that $L_0\cup L_1\cup\cdots\cup L_{i-1}\subset\cns(P)$. Moreover, there exists a rule $l_1,\dots,l_n\to l$ in $Q_i\cup P$ such that $\{l_1,\dots,l_n\}\subset L_0\cup L_1\cup\cdots\cup L_{i-1}$;
      the rule $l_1,\dots,l_n\to l$ is not in $P$, otherwise $l\in\cns(P)$; therefore $l_1,\dots,l_n\to l$ is a rule of $Q_i$. Moreover the rule $l_1,\dots,l_n\to l$ is exceptional in $Q$.  But $\cns(P)\subset\cns(Q)$, thus,  $Q$ is inconsistent, because $l_1,\dots,l_n,l\in\cns(Q)$. Contradicti{o}n.
      \end{proof}
     \begin{coro}\label{coro1}
     If $\cns(Q\circ_{rk}P)\cup Q$ is inconsistent,  then $Q\cup P$ is inconsistent.
     \end{coro}
     \begin{proof}
         If $Q\cup P$ is consistent, we have  $Q\circ_{rk}P=Q\cup P$. Thus,
     $\cns(\cns(Q\circ_{rk}P)\cup Q)=\cns(Q\cup P)\cup Q)=\cns(P\cup Q)$ the last equality because of Lema \ref{lem1}. 
     \end{proof}
     \begin{coro}\label{paraFP}
     If $Q\vdash P$ and $Q$ is consistent,  then $\cns(Q\circ_{rk}P)\cup Q$ is consistent.
     \end{coro}
     \begin{proof}
       If $\cns(Q\circ_{rk}P)\cup Q$ is inconsistent, by Corollary \ref{coro1}, we have $Q\cup P$ is inconsistent. Thus, all conditions of Lemma \ref{lem2}
      are verified. Therefore,  $\cns(Q\circ_{rk}P)\subset\cns(P)$.  Since $\cns(P)\subset\cns(Q)$, we have $\cns(Q\circ_{rk}P)\subset\cns(Q)$. Thus, because  $Q$ is consistent, $\cns(Q\circ_{rk}P)\cup Q$ is consistent, a contradicti{o}n.    
      \end{proof}
      
     Now we are ready to show that the operator defined by Equation~(\ref{def-merging}) satisfies {\myFP4}. Suppose that $P\cup P_1\cup P_2$ is consistent.
     In this case $\Delta_P(P_1,P_2)=\cns(P\cup P_1\cup P_2)$. Thus,
      \[\cns(\Delta_{P}(P_1,P_2)\cup P_2)=\cns(\cns(P\cup P_1\cup P_2)\cup P_2)=\cns(P\cup P_1\cup P_2)\]
     is consistent (last equality is given by Lemma \ref{lem1}).
     In this case we have
     \[\cns(\Delta_{P}(P_1,P_2)\cup P_2)=\cns(\Delta_{P}(P_1,P_2)\cup P_1)\]
     Suppose now $P\cup P_1\cup P_2$ is inconsistent.  By definition of the operator \[\Delta_P(P_1,P_2)=\cns(P_1\circ_{rk}P)\cap\cns(P_2\circ_{rk}P)\]
      If $\Delta_P(P_1,P_2)\cup P_2$ is inconsistent, it follows  $\cns(P_2\circ_{rk}P)\cup P_2$ is inconsistent, because \[\Delta_P(P_1,P_2)\cup P_2=(\cns(P_1\circ_{rk}P)\cap\cns(P_2\circ_{rk}P))\cup P_2\subset\cns(P_2\circ_{rk}P)\cup P_2\]
       and, by Corollary~\ref{coro1}, we have $P_2\cup P$ is inconsistent. Since $P_2$ is consistent and $P_2\vdash P$, the hypotheses of Lemma \ref{lem2} hold.
       Therefore,
        $P\vdash\cns(P_2\circ_{rk}P)$.  But $P_2\vdash P$, thus  $P_2\vdash\cns(P_2\circ_{rk}P)$. Therefore  $\cns(P_2\circ_{rk}P)\cup P_2$ is consistent, a contradiction.
        Thus,
     $\cns(P_2\circ_{rk}P)\cup P_2$ is consistent, that is Postulate  {\myFP4} holds.

     Actually, we have proved that if  $P_1\vdash P,\;P_2\vdash P$ and $P_1,\;P_2$ are consistent, then $\cns(P_2\circ_{rk}P)\cup P_2$ and $\cns(P_2\circ_{rk}P)\cup P_1$ are consistent.
          \\[3mm]
     Postulates {\myFP5-\myFP8} do not hold. We give counterexamples for each postulate.\\[3mm]
     {\bf Counterexample for  {\myFP5}}: $\Delta_{P}(\Phi_1)\cup\Delta_{P}(\Phi_2)\vdash\Delta_{P}(\Phi_1\sqcup\Phi_2)$.\par
      Consider $P=\{a\},\;P_1=\{a\fl b\},\;P_2=\{b\fl c\}$. Then $\Delta_{P}(P_1, P_2)=\{a,b,c\}$, $\Delta_{P}(P_1)=\{a,b\}$ y $\Delta_{P}(P_2)=\{a\}$. Thus, $\Delta_{P}(P_1)\cup\Delta_{P}(P_2)\not\vdash\Delta_{P}(P_1,P_2)$ because  $\Delta_{P}(P_1,P_2)\not\subset\Delta_{P}(P_1)\cup\Delta_{P}(P_2)$.\\[3mm]
     {\bf Counterexample for  {\myFP6}}: If $\Delta_{P}(\Phi_1)\cup\Delta_{P}(\Phi_2)$ is consistent,
     then
     $\Delta_{P}(\Phi_1\sqcup\Phi_2)\vdash\Delta_{P}(\Phi_1)\cup\Delta_{P}(\Phi_2)$.\par
     Consider $P=\{a\},\;P_1=\{\neg c\,;\,a\fl b\},\;P_2=\{b\fl c\}$.  Thus,  $\Delta_{P}(P_1)=\{a,b,\neg c\}$ and
     $\Delta_{P}(P_2)=\{a\}$. Therefore $\Delta_{P}(P_1)\cup\Delta_{P}(P_2)$ is consistent. Since $\cns(P_1\cup P_2\cup\ P)$ is inconsistent, we have $\Delta_{P}(P_1,P_2)=\cns(P_1\circ_{rk} P)\cap\cns(P_2\circ_{rk} P)=\{a\}$ and
    $\Delta_{P}(\Phi_1)\cup\Delta_{P}(\Phi_2)\not\subset\Delta_{P}(P_1,P_2)$.\\[3mm]
     {\bf Counterexample for   {\myFP7}}:  $\Delta_{P}(\Phi)\cup Q\vdash\Delta_{P\cup Q}(\Phi)$.\par
      Consider $P=\{c\},\;Q=\{a\},\;\Phi=\{a\fl b\}$. Thus, $\Delta_{P}(\Phi)\cup Q=\{a,c\}$ and $\Delta_{P\cup Q}(\Phi)=\{a,b,c\}$. Therefore $\Delta_{P}(\Phi)\cup Q\not\vdash\Delta_{P\cup Q}(\Phi)$.\\[3mm]
     {\bf Counterexample for  {\myFP8}}: If $\Delta_{P}(\Phi)\cup Q$ is consistent, then
     $\Delta_{P\cup Q}(\Phi)\vdash\Delta_{P}(\Phi)\cup Q$.\par
      Consider $\Phi= \Set H$ where $H=\{a\fl c\,;\,b\fl\neg c\},\;P=\{a\},\;Q=\{b\}$. We have $\Delta_P(\Phi)=\{a,c\}$, so $\Delta_{P}(\Phi)\cup Q=\{a,b,c\}$ is consistent.
       Moreover $\Delta_{P\cup Q}(\Phi)=\cns(\Phi\circ_{rk}(P\cup Q))$.  Since $P$ has no exceptional rules, we have  $\Phi_{\rho(\Phi,P\cup Q)}=\vac$. Therefore
     $\Delta_{P\cup Q}(\Phi)=\cns(\Phi\circ_{rk}(P\cup Q))=\cns(\Phi_{\rho(\Phi,P\cup Q)}\cup(P\cup Q))=\cns(P\cup Q)=\{a,b\}$. Thus,
      $\Delta_{P\cup Q}(\Phi)\not\vdash\Delta_{P}(\Phi)\cup Q$.
      \end{proof}  

     The rest of this section is devoted to the analysis of properties of merging operators defined in the style of Equation (\ref{def-merging}) but using hull revision operators and extended hull
     revision operators instead of rank revision. The operators are given in the following way:
     \begin{eqnarray}\label{merging-h}
     \Delta_P^{h}(P_1,\dots,P_n)=\left\{
                                    \begin{array}{ll}
                                      \cns(P\cup P_1\dots\cup P_n), & \hbox{if consistent;} \\
                                      \bigcap\cns(P_i\circ_{h} P), & \hbox{otherwise.}
                                    \end{array}
                                  \right.
     \end{eqnarray}
     \begin{eqnarray}\label{merging-eh}
     \Delta_P^{eh}(P_1,\dots,P_n)=\left\{
                                    \begin{array}{ll}
                                      \cns(P\cup P_1\dots\cup P_n), & \hbox{if consistent;} \\
                                      \bigcap\cns(P_i\circ_{eh} P), & \hbox{otherwise.}
                                    \end{array}
                                  \right.
     \end{eqnarray}
     
     The following theorem tells us about the behavior of the merging operators defined by Equations (\ref{merging-h}) and  (\ref{merging-eh}).
     
     \begin{teo}
The operators $\Delta^{h}$ and $\Delta^{eh}$satisfy the postulates \myFP$0$, \myFP$1$ and \myFP$4$. They don't satisfy \myFP$3$-\myFP$8$
\end{teo}
\begin{proof}
         Since $Q\circ_{eh} P\vdash Q\circ_{h} P\vdash Q\circ_{rk} P$,  the verification of Postulate   {\myFP0}
         by these operators  is straightforward. Postulates  {\myFP1} and  {\myFP2} follow from  definitions of hull revision and extended hull revision.

          The given counterexamples for  {\myFP3}, {\myFP5}, {\myFP6},  {\myFP7} and  {\myFP8} in the case of merging using rank revision remain valid for our merging defined by Equations (\ref{merging-h}) and (\ref{merging-eh}), because  in the counterexamples  rank revision, hull revision and extended hull revision coincide.
          Notice that in the counterexample for {\myFP8}, we have $I_{\Phi}(P\cup Q)=\{\{a\fl c\},\{b\fl\neg c\}\}$.  Thus $h_{\Phi}(P\cup Q)=\vac$, and therefore
        \[\Delta_{P\cup Q}^{h}(\Phi)=\cns(h_{\Phi}(P\cup Q)\cup(P\cup Q))=\cns(P\cup Q)=\{a,b\}\]
        We have also  $\Delta_{P}^{h}(\Phi)\cup Q=\Delta_{P}(\Phi)\cup Q=\{a,b,c\}$.
        Thus, $\Delta_{P\cup Q}^{h}(\Phi)\not\vdash\Delta_{P}^{h}(\Phi)\cup Q$.
        We have $\Phi\circ_{eh}(P\cup Q)=\langle\{a:b:a\fl c\},\{a:b:b\fl\neg c\}\rangle$ thus
        \[\cns(\Phi\circ_{eh}(P\cup Q))=\cns(\{a:b:a\fl c\})\cap\cns(\{a:b:b\fl\neg c\})=\{a,b,c\}\cap\{a,b,\neg c\}=\{a,b\}\]
        Therefore
        $\Delta_{P\cup Q}^{eh}(\Phi)=\cns(\Phi\circ_{eh}(P\cup Q))=\{a,b\}$.
        Thus, in this example, we have
         $\Delta_{P\cup Q}(\Phi)=\Delta_{P\cup Q}^{h}(\Phi)=\Delta_{P\cup Q}^{eh}(\Phi)=\{a,b\}$
         and
         $\Delta_{P}(\Phi)\cup Q=\Delta_{P}^{h}(\Phi)\cup Q=\Delta_{P}^{eh}(\Phi)\cup Q=\{a,b,c\}$
         Therefore
         $\Delta_{P\cup Q}^{eh}(\Phi)\not\vdash\Delta_{P}^{eh}(\Phi)\cup Q$.
         \bigskip

           Postulate (\myFP4) is also problematic for the new operators. We give a counterexample for this postulate. The counterexample works in the two cases  (hull revision and
           extended hull revision) because the two operators coincide.\\
         Consider
          $P_1=\{a;a\fl d;a,d\fl c\}$, 
          $P_2=\{a;b;c\fl\neg b;a\fl d\}$ and 
          $P=\{a;b\fl c;d\fl c\}$.
          We prove that they verify the hypothesis of Postulate   (\myFP4). Moreover $\Delta_P^{\star}(P_1,P_2)\cup P_1$ is consistent, whereas $\Delta_P^{\star}(P_1,P_2)\cup P_2$ is inconsistent, where $\star\in\{h,eh\}$.\\
         $\cns(P_1)=\{a,d,c\}$, $\cns(P_2)=\{a,b,d\}$ and $\cns(P)=\{a\}$; thus, $P_1\vdash P,\,P_2\vdash P$ and $P_1$ and $P_2$ are consistent. We compute $\Delta_P^{h}(P_1,P_2)$.
         Since $P_1\cup P_2\cup P$ is inconsistent (more precisely $P_2\cup P$ is inconsistent) we have, by definition
         \[\Delta_P^{h}(P_1,P_2)=\cns(P_1\circ_{h}P)\cap\cns(P_2\circ_{h}P)\]
         But $P_1\cup P$ is consistent, thus
         \[\cns(P_1\circ_{h}P)=\cns(P_1\cup P)=\{a,d,c\}\]
          Since $P_2\cup P$ is inconsistent, we look for the exceptional rules of $P_2$. We see that $c\fl\neg b$ is the unique exceptional rule of $P_2$. This rule is $P$-consistent. Now we compute $I_{P_2}(P)$, the maximal subsets of  $P_2$ containing $c\fl\neg b$ which are  $P$-consistent; we observe that if  $b,\,c\fl\neg b\in M$, with $M\subset P_2$, necessarily $M\cup P$ is inconsistent (because by forward chaining we obtain $\neg b$). Thus we conclude that in the sets of $I_{P_2}(P)$ the element $b$ does not appear.
          Notice that $\{a;c\fl\neg b;a\fl d\}$ is the subset of $P_2$ such that  it doesn't contain $b$ and moreover is $P$-consistent. Thus we  conclude
         \[I_{P_2}(P)=\langle\{a;c\fl\neg b;a\fl d\}\rangle\]
         That is, there exists a unique maximal subset of  $P_2$ containing  $b\fl\neg c$ and  $P$-consistent.\\
        Since $h_{P_2}(P)=\cap I_{P_2}(P)=\{a;c\fl\neg b;a\fl d\}$ and
         $P_2\circ_{h}P=h_{P_2}(P)\cup P$
         we have
         $\cns(P_2\circ_{h}P)=\cns(h_{P_2}(P)\cup P)=\{a,d,c\}$.
         Therefore
         $\Delta_P^{h}(P_1,P_2)=\cns(P_1\circ_{h}P)\cap\cns(P_2\circ_{h}P)=\{a,d,c\}$.
         Finally we have
         $\cns(\Delta_P^{h}(P_1,P_2)\cup P_1)=\{a,d,c\}$
         and
         $\cns(\Delta_P^{h}(P_1,P_2)\cup P_2)=\lit$
         because $b\in P_2$ and,  since $c\fl\neg b\in P_2$ and $c\in \Delta_P^{h}(P_1,P_2)$ we obtain by forward chaining $\neg b$.\\
         Since
         $\Delta_P^{eh}(P_1,P_2)=\cns(P_1\circ_{eh}P)\cap\cns(P_2\circ_{eh}P)$
         and
         $\cns(P_1\circ_{h}P)=\cns(P_1\circ_{eh}P)\qquad\cns(P_2\circ_{h}P)=\cns(P_2\circ_{eh}P)$
         we have
         $\Delta_P^{h}(P_1,P_2)=\Delta_P^{eh}(P_1,P_2)$.
         This shows that postulate  (\myFP4) doesn't hold for merging operators defined with extended hull revision.
         \end{proof}
         
    \section{Concluding remarks}
     We have defined three syntactic arbitration operators. They satisfy Postulates
     $(\mbox{\sf SA}1)-(\mbox{\sf SA}4)$ and $(\mbox{\sf SA}7)-(\mbox{\sf SA}8)$ but they fail to satisfy Postulates
     $(\mbox{\sf SA}5)-(\mbox{\sf SA}6)$. These operators seem to be a good compromise between rationality and computation, good properties and tractability.
     The operator $\dm_{rk}$ is polynomial because $\circ_{rk}$ is clearly polynomial. The other operators have a higher computational complexity
     inherent to computation of maximal consistent sets.

Concerning merging with integrity constraints, we have also three syntactic operators. The operator defined using rank revision satisfies Postulates
(\myFP0), (\myFP1), (\myFP2) and (\myFP4), whereas the merging operators defined using hull revision and extended hull revision satisfy only Postulates
(\myFP0),  (\myFP1) and (\myFP2). Thus in the case of merging with integrity constraints the operator $\Delta$ defined with rank revision is better than
the two other operators both from the point of view of rational behavior as from the point of view of computational complexity.

In order to satisfy more rational properties, it will be interesting to explore other alternatives to definition given by Equation (\ref{def-merging}), in particular,  the use of some kinds of aggregation functions (like majority)
 to produce the  resulting facts(see for instance \cite{Kon00}) could lead to operators having a better behavior.
It remains also to explore the exact relationships between properties of arbitration and merging operators defined in Equations (\ref{eq-arbi}), (\ref{def-merging}), (\ref{merging-h}) and (\ref{merging-eh}) and the properties of revision operators used in those definitions.

\bibliographystyle{plain}

\end{document}